\documentclass[sigconf]{acmart}

\usepackage[utf8]{inputenc} 
\usepackage[T1]{fontenc}    
\usepackage{hyperref}       
\usepackage{url}            
\usepackage{booktabs}       
\usepackage{amsfonts}       
\usepackage{nicefrac}       
\usepackage{microtype}      
\usepackage{xcolor}         

\usepackage{graphicx}
\usepackage{subcaption}

\usepackage{amsmath}
\usepackage{amsthm}
\usepackage{bm}

\usepackage[ruled,vlined]{algorithm2e}

\usepackage{multirow}
\usepackage{enumitem}

\usepackage[normalem]{ulem}
\useunder{\uline}{\ul}{}

\usepackage{listings}
\usepackage{fancyvrb}

\usepackage{cleveref}
\crefformat{section}{\S#2#1#3} 
\crefformat{subsection}{\S#2#1#3}
\crefformat{subsubsection}{\S#2#1#3}

\DeclareMathOperator*{\argmax}{arg\,max}
\DeclareMathOperator*{\argmin}{arg\,min}

\newtheorem*{theorem*}{Theorem}
\newtheorem{theorem}{Theorem}[section]

\newtheorem{lemma}[theorem]{Lemma}
\newtheorem{proposition}[theorem]{Proposition}

\newcommand{\ltwo}{$\ell_2$}

\newcommand{\x}{\mathbf{x}}
\newcommand{\y}{\mathbf{y}}
\newcommand{\wt}{\widehat{\theta}}
\newcommand{\fxt}{f(\x;\wt)}
\newcommand{\faxt}{f(a(\x);\wt)}
\newcommand{\fxit}{f(\x_i;\wt)}
\newcommand{\faxit}{f(a(\x_i);\wt)}
\newcommand{\ma}{\mathcal{A}}
\newcommand{\mP}{\mathbf{P}}
\newcommand{\mI}{\mathbb{I}}
\newcommand{\rpp}{r_{\mP'}}
\newcommand{\rpa}{r_{\mP, \ma}}

\newcommand{\wrpp}{\widehat{r}_{\mP'}}
\newcommand{\rp}{r_\mP}
\newcommand{\wrp}{\widehat{r}_{\mP}}

\newcommand{\X}{\mathbf{X}}
\newcommand{\Y}{\mathbf{Y}}

\newcommand{\base}{\textsf{B}}
\newcommand{\va}{\textsf{VA}}
\newcommand{\ra}{\textsf{RVA}}
\newcommand{\vwa}{\textsf{VWA}}
\newcommand{\rwa}{\textsf{RWA}}

\setlist[itemize]{leftmargin=*}
\AtBeginDocument{%
  \providecommand\BibTeX{{%
    \normalfont B\kern-0.5em{\scshape i\kern-0.25em b}\kern-0.8em\TeX}}}

\copyrightyear{2022}
\acmYear{2022}
\setcopyright{rightsretained}
\acmConference[KDD '22]{Proceedings of the 28th ACM SIGKDD Conference on Knowledge Discovery and Data Mining}{August 14--18, 2022}{Washington, DC, USA}
\acmBooktitle{Proceedings of the 28th ACM SIGKDD Conference on Knowledge Discovery and Data Mining (KDD '22), August 14--18, 2022, Washington, DC, USA}
\acmDOI{10.1145/3534678.3539438}
\acmISBN{978-1-4503-9385-0/22/08}

\begin{document}

\title{Toward Learning Robust and Invariant Representations with Alignment Regularization and Data Augmentation}

\author{Haohan Wang}
\affiliation{%
  \institution{Carnegie Mellon University}
  \city{Pittsburgh}
  \state{PA}
  \country{USA}
}
\email{haohanw@cs.cmu.edu}

\author{Zeyi Huang}
\affiliation{%
  \institution{University of Wisconsin-Madison}
  \city{Madison}
  \country{WI}
  \country{USA}
}
\email{zeyih@andrew.cmu.edu}

\author{Xindi Wu}
\affiliation{%
  \institution{Carnegie Mellon University}
  \city{Pittsburgh}
  \state{PA}
  \country{USA}
}
\email{xindiw@andrew.cmu.edu}

\author{Eric P. Xing}
\affiliation{%
  \institution{Carnegie Mellon University}
  \city{Pittsburgh}
  \state{PA}
  \country{USA}
}
\email{epxing@cs.cmu.edu}

\renewcommand{\shortauthors}{Wang, et al.}

\begin{abstract}
  Data augmentation has been proven to be an effective
technique for developing machine learning models 
that are 
robust to known classes of distributional shifts 
(\textit{e.g.}, rotations of images),
and alignment regularization is a technique often 
used together with data augmentation to further 
help the model learn 
representations invariant to the shifts used 
to augment the data. 
In this paper, 
motivated by a proliferation of options of alignment regularizations, 
we seek to evaluate the performances of several popular design choices along the dimensions of robustness and invariance, for which we introduce a new test procedure. 
Our synthetic experiment results
speak to the benefits of 
squared $\ell_2$ norm regularization.
Further, 
we also formally analyze the behavior 
of alignment regularization
to complement our empirical study 
under assumptions we consider realistic. 
Finally, we test this simple technique we identify 
(worst-case data augmentation with squared $\ell_2$ norm alignment regularization) 
and show that the benefits of this method outrun 
those of the specially designed methods. 
We also release a software package 
in both TensorFlow and PyTorch 
for users to use the method 
with a couple of lines\footnote{ \href{https://github.com/jyanln/AlignReg}{https://github.com/jyanln/AlignReg}}. 
\end{abstract}




\maketitle

\section{Introduction}

Data augmentation, 
\textit{i.e.}, to increase the dataset size through 
generating new samples by 
transforming the existing samples with some predefined functions, 
is probably one of the most often used techniques to improve 
a machine learning model's performance. 
It has helped machine learning models achieve high prediction accuracy 
over various benchmarks 
\citep[\textit{e.g.,}][]{LimKKKK19,HoLCSA19,Zhong0KL020,BuslaevIKPDK20,Ghiasi20,Kostrikov21}. 

In addition to improving prediction accuracy, 
the community has also leveraged data augmentation to help the models 
learn more robust representations that can generalize 
to the datasets distributed differently 
\citep[\textit{e.g.,}][]{zheng2016improving,SunYOHX18,HuangWXH20,MinMDPL20}. 
To improve robustness, the community usually designed the transformation functions used to augment the data in correspondence to the transformations we see in the real world \citep{hernandez2020data}, 
such as the changes of image texture or contrast. 
Thus, models trained with these augmented data 
are more likely to be invariant to these designed transformations, 
such as the texture or contrast variations of the input images. 

To further help a model learn representations invariant to the transformations, 
we can regularize the model so that the distance between representations learned by the model 
from a pair of data (the original one and the transformed counterpart) will be small.
This regularization has been used extensively recently to help models learn more robust and invariant representations \citep[\textit{e.g.,}][]{liang2018learning,kannan2018adversarial,ZhangYJXGJ19,Hendrycks2020augmix}. 
Motivated by this popularity, 
this paper mainly studies the behaviors of this regularization, which we refer to as 
\emph{alignment regularization} 
(AR).
In particular, we seek to answer the question:
\emph{how should we use alignment regularization to take advantage of the augmented data to the fullest extent to learn robust and invariant models?}

To answer this, 
we first conduct a range of experiments 
over image classification benchmarks 
to evaluate how popular variants of AR
contribute to learning robust and invariant models. 
We test for accuracy, robustness, and invariance, 
for which we propose a new test procedure.
Our empirical study favors 
the squared \ltwo{} norm. 
Our contributions of this paper are as follows. 

\begin{itemize}
\item With a new invariance test, we show 
that 
alignment regularization is important
to help the model 
learn representations 
invariant to the transformation function, 
and squared \ltwo{} norm is considered  
the favorable choice as assessed by 
a variety of empirical evaluations (Section~\ref{sec:motivate}).  
\item We formalize a generalization error bound 
for models trained with AR and augmented data
(Section~\ref{sec:method}). 
\item We test the method we identified
(squared \ltwo{} norm as AR) 
in multiple scenarios.
We notice that this generic approach can compete with methods specially designed for different scenarios, which we believe endorses its empirical strength
(Section~\ref{sec:exp}). 
\end{itemize}

\section{Related Work and Key Differences}
\label{sec:related}
Tracing back to the earliest convolutional neural networks, \citep{wang2017origin}, 
we notice that even early models for the MNIST dataset 
have been boosted by data augmentation \citep{abu1990learning,simard1991tangent,lecun1998gradient}. 
Later, the rapidly growing machine learning community 
has seen a proliferate development of data augmentation techniques 
that have helped models climb the state-of-the-art ladder \citep{shorten2019survey}. 
Among the augmentation techniques, 
the most relevant one to this paper is to
generate the samples (with constraint) 
that maximize the training loss along with training \citep{FawziSTF16}.

While the above paragraph mainly discusses 
how to generate the augmented samples, 
we mainly study
how to train the models with augmented samples. 
For example, instead of directly mixing augmented samples 
with the original samples, 
one can consider regularizing the representations (or outputs)
of original samples and augmented samples 
to be close under a distance metric (which we refer to as alignment regularization, AR). 
Many concrete ideas have been explored in different contexts. 
For example, $\ell_2$ distance and cosine similarities
between internal representations in speech recognition
\citep{liang2018learning}, 
squared $\ell_2$ distance between logits \citep{kannan2018adversarial},
or KL divergence between softmax outputs \citep{ZhangYJXGJ19} in adversarially robust vision models, 
Jensen–Shannon divergence (of three distributions) 
between embeddings for texture invariant image classification \citep{Hendrycks2020augmix}. 
These are but a few highlights of the concrete and successful implementations for different applications
out of a vast collection (\textit{e.g.}, \citep{WuMWZGXX19,ZFY019, zhang2019regularizing, Shah_2019_CVPR, asai2020logicguided,sajjadi2016regularization,zheng2016improving,xie2015hyper}), 
and we can expect methods permuting these three elements (distance metrics, representation or outputs, and applications) 
to be discussed in the future. 
Further, 
given the popularity of GAN \citep{goodfellow2016nips} 
and domain adversarial neural network \citep{ganin2016domain}, 
we can also expect the distance metric generalizes 
to a specialized discriminator (\textit{i.e.}, a classifier), 
which can be intuitively understood as a calculated
(usually maximized) distance measure,
and one example here is the 
Wasserstein-1 metric \citep{arjovsky2017wasserstein,gulrajani2017improved}. 

\textbf{Key Differences:}
With this rich collection of regularizing choices, 
which one method should we consider in general? 
More importantly, 
do we need the regularization at all? 
These questions are important 
for multiple reasons,
especially since 
sometimes 
AR may worsen the results \citep{jeong2019consistency}. 
In this paper, we first conduct an empirical study to 
show that 
AR 
(especially squared \ltwo{} norm) 
can help 
learn robust and invariant models,
we then also derive generalization error bounds to 
complement our empirical conclusion. 


There are also several previous discussions 
regarding the detailed understandings of data augmentation \citep{yang2019invariance, chen2019grouptheoretic,hernndezgarca2018data,rajput2019does,DaoGRSSR19,ghosh2021on,zhang2021how}, 
among which, \citep{yang2019invariance} is probably the most relevant 
as it also defends the usage of the AR. 
In addition to what is reported in \citep{yang2019invariance}, our work also connects to invariance
and shows that another advantage of AR is to learn invariant representations. 




\section{Accuracy, Robustness, and Invariance}
\label{sec:background}
This section discusses the three major evaluation metrics we will use to test AR. 
We will first recapitulate the background of accuracy and robustness. Then we will introduce our definition of invariance and our proposed evaluation.  

\textbf{Notations}
$(\X, \Y)$ denotes the data,
where $\X \in \mathcal{R}^{n\times p}$ and $\Y \in \{0, 1\}^{n\times k}$
(one-hot vectors for $k$ classes).  
$(\x, \y)$ denotes a sample. 
$f(, \theta)$ denotes the model,
which takes in the data and outputs the softmax (probabilities of the prediction) 
and $\theta$ denotes the corresponding parameters. 
$g()$ completes the prediction (\textit{i.e.},
mapping softmax to one-hot prediction). 
$a()$ denotes a function for data augmentation, 
\textit{i.e.},
a transformation function. 
$a\in \ma$, 
which is the set of transformation functions of interest. 
$\mP$ denotes the distribution of $(\x, \y)$.
For any sampled $(\x, \y)$, we can have $(a(\x), \y)$, 
and we use $\mP_a$ to denote the distribution of these transformed samples. 
Further, we use $\mathbf{Q}_{a(\x), \wt}$ to 
denote the distribution of $\faxt$ for $(\x, \y)\sim \mP$. 
$D(\cdot, \cdot)$ is a distance measure over two distributions. 
$r(\theta)$ denotes the risk of model $\theta$. 
$\widehat{\cdot}$ denotes the estimation of the term $\cdot$, 
e.g.,
$\wrp(\wt)$ denotes the empirical risk of the estimated model.   

\subsection{Accuracy}
The community studying the statistical property
of the error bound usually focuses on the expected risk defined as
\begin{align}
    \rp(\wt) =  \mathbb{E}_{(\x, \y) \sim \mP} \mI[g(\fxt) \neq \y], 
    \label{eq:accuracy}
\end{align}
where $\mI[\mathbf{c}]$ is a function that returns $1$ if the condition $\mathbf{c}$ holds. 

In practice, the error is evaluated by replacing $\mP$ with a hold-out test dataset, and the accuracy is $1 - \rp(\wt)$. 

\subsection{Robustness}
We define 
robustness as 
the worst-case expected risk
when the test data is allowed to 
be transformed  
by functions in $\ma$, following \citep[\textit{e.g.},][]{szegedy2013intriguing,goodfellow2014explaining}.
Formally, we study the worst-case error as 
\begin{align}
    \rpa(\wt) =  \mathbb{E}_{(\x, \y) \sim \mP} \max_{a\sim \ma}\mI[g(\faxt) \neq \y],
    \label{eq:robustness}
\end{align}
where we use $\rpa(\wt)$ to denote the robust error as it will depend on $\ma$. 
In practice, 
the robust error is also evaluated 
by replacing $\mP$ with a hold-out dataset.

\subsection{Invariance}
\label{sec:background:invariance}
Further, 
high robustness performances do not necessarily mean 
the model is truly invariant to the transformation functions \citep{hernandez2019learning}, 
and we continue to introduce a new test 
to evaluate the model's behavior in learning representations invariant to the transformations.

\paragraph{Invariance}
If the model can learn a representation invariant to the transformation functions,
it will map the samples of different transformations
to the same representation. 
Intuitively, to measure how invariant a model is to the transformations in $\ma$, 
we can calculate the distances between each pair of the two transformed samples when a sample is transformed with functions in $\ma$. 
Thus, we define the following term to measure invariance:
\begin{align}
    I_{\mP, \ma}(\wt) = \sup_{a_1, a_2\in \ma} D (\mathbf{Q}_{a_1(\x), \wt}, \mathbf{Q}_{a_2(\x), \wt}), 
    \label{eq:invariance}
\end{align}
We suggest using
Wasserstein metric as $D(\cdot, \cdot)$,
considering its favorable properties 
(\textit{e.g.}, see practical examples in Figure 1 of
\citep{cuturi2014fast} or theoretical discussions in \citep{villani2008optimal}).

In practice, we also need to replace $\mP$ 
with a hold-out dataset
so that the evaluation can be performed. 
In addition, we notice that
$I_{\mP, \ma}(\wt)$, 
although intuitive, 
is not convenient in practice because the evaluated values are not bounded. 
Thus, we reformulate it into the following invariance test procedure, 
whose final score will be bounded between 0 and 1 (the higher, the better). 
Therefore, the score can be conveniently discussed together with accuracy and robust accuracy, which are also bounded between 0 and 1.

\paragraph{Invariance test}
Given a family of transformation functions
used in data augmentation $\ma = \{a_1(), a_2(), \dots, a_t()\}$ 
of $t$ elements, 
and a collection of samples (from the hold-out dataset) of the same label $i$, 
denoted as $\X^{(i)}$, 
the evaluation procedure is as follows. 
We first generate the transformed copies of $\X^{(i)}$ with $\ma$, 
resulting in $\X^{(i)}_{a_1}, \X^{(i)}_{a_2}, \dots, \X^{(i)}_{a_t}$. 
We combined these copies into a dataset, denoted as $\mathcal{X}^{(i)}$. 
For every sample $\x$ in $\mathcal{X}^{(i)}$, 
we retrieve its $t$ nearest neighbors of other samples in $\mathcal{X}^{(i)}$, 
and calculate the overlap of the retrieved samples 
with the transformed copies of $\x$ by $\ma$, 
\textit{i.e.}, 
$\{a_1(\x), a_2(\x), \dots, a_t(\x)\}$. 
The calculated overlap score will be in $[0, 1]$ in general,
but since the identity map is usually in $\ma$, 
this score will usually be in $[1/t, 1]$. 

During the retrieval of nearest neighbors, 
we consider the distance function of the two samples (namely $\x$ and $\x'$)
as 
$$d(f(\x;\wt),f(\x';\wt))$$,
where $\wt$ is the model we are interested in examining. 
In the empirical study later, 
we consider $d(\mathbf{u},\mathbf{v})=\Vert\mathbf{u}-\mathbf{v}\Vert_1$, with $\mathbf{u}$ and $\mathbf{v}$ denoting two vectors.
If we use other distance functions, 
the reported values may differ, but we notice that
the rank of the methods compared in terms of 
this test barely changes. 

Finally, we iterate through label $i$ and report the averaged score for all the labels as the final score. 
A higher score 
indicates the model $\wt$ is more invariant to 
the transformation functions in $\ma$. 

This invariance test procedure is formally presented in Algorithm~\ref{alg:invariance} below. 

\begin{algorithm}
\scriptsize 
\SetAlgoLined
\KwResult{$\widehat{I}(\wt)$}
\textbf{Input:} a family of transformation functions $\ma = \{a_1(), a_2(), \dots, a_t()\}$, 
a hold-out dataset $(\X,\Y)$, the model of interest $\wt$, and a distance metric $d()$\;
 \For {every label $i$ }{
    identify all the samples from $(\X,\Y)$ with label $i$, name this set of samples $\X^{(i)}$\;
    \For {every $a() \in \ma$}{
        generate $\X_a^{(i)}$ by applying $a()$ to every $\x \in \X^{(i)}$\;
    }
    $\mathcal{X}^{(i)} = \X^{(i)}_{a_1} \cup \X^{(i)}_{a_2} \cup, \dots, \cup \X^{(i)}_{a_t}$\;
    \For {every $\x \in \X^{(i)}$}{
        generate set $T = \{a_1(\x), a_2(\x), \dots, a_t(\x)\}$\;
        \For {every $\x' \in \mathcal{X}^{(i)}$}{
            calculate the distance between $\x$ and $\x'$ with $d(f(\x;\wt),f(\x';\wt))$\;
        }
        retrieve the $t$ nearest neighbors of $\x$ out of $\mathcal{X}^{(i)}$, and name this set of samples $K$\;
        calculate the score for $\x$ with $\vert T \cap K\vert/\vert T\vert$, where $\vert T\vert$ denotes the cardinality of the set $T$\;
    }
    calculate the score for label $i$ as the average score across all $\x \in \X^{(i)}$\;
 }
 calculate the final score $\widehat{I}(\wt)$ as the average score across all the labels\;
 \caption{Invariance test}
 \label{alg:invariance}
\end{algorithm} 


\section{Empirical Study}
\label{sec:motivate}
In this section, 
we conduct experiments to 
study the relationship between robustness and invariance, 
as well as how training with AR
can help improve the invariance score. 
In short, our empirical study in this section
will lead us to the following three major conclusions:
\begin{itemize}
    \item High robust accuracy does not necessarily mean a high invariance score and vice versa. 
    \item AR can help improve the invariance score. 
    \item Squared \ltwo{} norm over logits is considered the empirically most favorable AR option for learning robust and invariant representations. 
\end{itemize}

\begin{table*}[t]
\small 
\centering 
\caption{Test results on MNIST dataset for different ARs over three evaluation metrics and three distribution shifts. 
\textsf{B} denotes Baseline, i.e., the model does not use any data augmentation; 
\textsf{V} denotes vanilla augmentation, i.e., the model uses data augmentation but not AR; 
\textsf{L} denotes $\ell_1$ norm;
\textsf{S} denotes squared $\ell_2$ norm;
\textsf{C} denotes cosine similarity;
\textsf{K} denotes KL divergence;
\textsf{W} denotes Wasserstein-1 metric;
\textsf{D} denotes GAN discriminator.
}
\setlength\tabcolsep{5pt}
\begin{tabular}{cccccccccc}
\hline
 &  & \textsf{B} & \textsf{V} & \textsf{L} & \textsf{S} & \textsf{C} & \textsf{K} & \textsf{W} & \textsf{D} \\ \hline
\multirow{3}{*}{Texture} & \textsf{Accuracy} & $99.2_{\pm0.0}$ & $99.2_{\pm0.0}$ & $99.0_{\pm0.1}$ & $99.1_{\pm0.0}$ & \bm{$99.4_{\pm0.0}$} & $68.7_{\pm41}$ & $98.7_{\pm0.1}$ & $99.1_{\pm0.1}$ \\
 & \textsf{Robustness} & $98.3_{\pm0.3}$ & $99.0_{\pm0.0}$ & $99.0_{\pm0.0}$ & $99.0_{\pm0.0}$ & \bm{$99.1_{\pm0.0}$} & $68.7_{\pm41}$ & $98.4_{\pm0.1}$ & $98.8_{\pm0.1}$ \\
 & \textsf{Invariance} & $92.4_{\pm0.0}$ & $99.2_{\pm0.0}$ & \bm{$100_{\pm0.0}$} & \bm{$100_{\pm0.0}$} & $99.0_{\pm0.0}$ & $76.0_{\pm34}$ & $60.7_{\pm2.9}$ & $35.0_{\pm6.7}$ \\ \hline
\multirow{3}{*}{Rotation} & \textsf{Accuracy} & $99.2_{\pm0.0}$ & $99.0_{\pm0.1}$ & \bm{$99.3_{\pm0.0}$} & \bm{$99.3_{\pm0.0}$} & $99.0_{\pm0.0}$ & $98.8_{\pm0.0}$ & $98.5_{\pm0.4}$ & $98.9_{\pm0.0}$ \\
 & \textsf{Robustness} & $28.9_{\pm0.6}$ & $93.6_{\pm0.3}$ & \bm{$95.2_{\pm0.1}$} & $95.1_{\pm0.1}$ & $93.5_{\pm0.1}$ & $94.5_{\pm0.2}$ & $92.3_{\pm0.8}$ & $93.2_{\pm0.7}$ \\
 & \textsf{Invariance} & $20.6_{\pm0.4}$ & $58.3_{\pm2.2}$ & $66.0_{\pm3.8}$ & $65.4_{\pm3.5}$ & $29.1_{\pm0.6}$ & \bm{$71.9_{\pm2.8}$} & $48.7_{\pm1.9}$ & $39.3_{\pm6.9}$ \\ \hline
\multirow{3}{*}{Contrast} & \textsf{Accuracy} & $99.2_{\pm0.0}$ & $98.9_{\pm0.3}$ & \bm{$99.4_{\pm0.0}$} & \bm{$99.4_{\pm0.0}$} & $99.2_{\pm0.0}$ & $98.9_{\pm0.0}$ & $98.7_{\pm0.0}$ & $99.1_{\pm0.0}$ \\
 & \textsf{Robustness} & $26.0_{\pm1.0}$ & $95.4_{\pm2.6}$ & $96.8_{\pm0.8}$ & $97.4_{\pm0.6}$ & \bm{$97.9_{\pm0.4}$} & $88.4_{\pm4.5}$ & $87.2_{\pm9.6}$ & $97.7_{\pm0.6}$ \\
 & \textsf{Invariance} & $20.7_{\pm1.1}$ & $37.5_{\pm6.9}$ & \bm{$41.4_{\pm0.3}$} & $41.3_{\pm0.4}$ & $26.3_{\pm1.1}$ & $40.3_{\pm0.9}$ & $28.4_{\pm1.7}$ & $20.0_{\pm0.1}$ \\ \hline
\end{tabular}
\label{tab:mnist:consistency}
\end{table*}

\begin{table*}[t]
\small 
\centering 
\caption{Test results on CIFAR10 dataset for different ARs over three evaluation metrics and three distribution shifts. 
Notations are the same as in Table~\ref{tab:mnist:consistency}.
}
\setlength\tabcolsep{5pt}
\begin{tabular}{cccccccccc}
\hline
 &  & \textsf{B} & \textsf{V} & \textsf{L} & \textsf{S} & \textsf{C} & \textsf{K} & \textsf{W} & \textsf{D} \\ \hline
\multirow{3}{*}{Texture} & \textsf{Accuracy} & \bm{$88.5_{\pm1.7}$} & $86.3_{\pm0.3}$ & $82.8_{\pm0.4}$ & $82.0_{\pm0.0}$ & $86.8_{\pm0.1}$ & $84.6_{\pm0.4}$ & $86.8_{\pm0.1}$ & $86.5_{\pm0.4}$ \\
 & \textsf{Robustness} & $38.3_{\pm0.7}$ & $76.5_{\pm0.0}$ & $79.1_{\pm0.2}$ & \bm{$79.4_{\pm0.1}$} & $76.8_{\pm0.1}$ & $75.6_{\pm0.1}$ & $76.8_{\pm0.1}$ & $77.3_{\pm0.2}$ \\
 & \textsf{Invariance} & $44.7_{\pm0.5}$ & $94.1_{\pm0.6}$ & \bm{$100_{\pm0.0}$} & \bm{$100_{\pm0.0}$} & $94.2_{\pm1.2}$ & $96.7_{\pm1.0}$ & $93.4_{\pm0.4}$ & $95.4_{\pm0.8}$ \\ \hline
\multirow{3}{*}{Rotation} & \textsf{Accuracy} & \bm{$88.5_{\pm1.7}$} & $81.1_{\pm2.3}$ & $80.3_{\pm4.7}$ & $78.1_{\pm2.2}$ & $80.7_{\pm1.8}$ & $80.4_{\pm6.8}$ & $87.5_{\pm1.6}$ & $83.7_{\pm4.0}$ \\
 & \textsf{Robustness} & $15.3_{\pm1.1}$ & $49.0_{\pm0.4}$ & $50.7_{\pm2.6}$ & \bm{$51.1_{\pm1.0}$} & $47.0_{\pm0.6}$ & $46.6_{\pm4.5}$ & $49.4_{\pm0.9}$ & $47.7_{\pm1.7}$ \\
 & \textsf{Invariance} & $47.3_{\pm0.6}$ & $54.6_{\pm2.0}$ & $40.7_{\pm1.4}$ & $55.2_{\pm1.6}$ & \bm{$55.7_{\pm1.1}$} & $55.5_{\pm1.0}$ & \bm{$55.7_{\pm1.3}$} & $54.5_{\pm0.7}$ \\ \hline
\multirow{3}{*}{Contrast} & \textsf{Accuracy} & $88.5_{\pm1.7}$ & $85.2_{\pm4.5}$ & $88.8_{\pm1.3}$ & $86.9_{\pm2.0}$ & $89.6_{\pm0.9}$ & $83.4_{\pm3.7}$ & $87.2_{\pm2.2}$ & \bm{$89.7_{\pm0.9}$} \\
 & \textsf{Robustness} & $54.5_{\pm0.8}$ & $77.7_{\pm0.8}$ & $83.1_{\pm1.3}$ & \bm{$83.4_{\pm1.1}$} & $80.7_{\pm2.4}$ & $79.5_{\pm2.9}$ & $82.8_{\pm1.0}$ & $80.8_{\pm5.6}$ \\
 & \textsf{Invariance} & $53.1_{\pm1.6}$ & $67.5_{\pm0.1}$ & $69.8_{\pm4.8}$ & \bm{$71.3_{\pm2.3}$} & $53.6_{\pm4.7}$ & $73.0_{\pm2.2}$ & $74.6_{\pm0.8}$ & $66.6_{\pm4.3}$ \\ \hline
\end{tabular}
\label{tab:cifar10:consistency}
\end{table*}

\subsection{Experiment Setup}
\label{sec:exp:sync}

Our empirical investigation is conducted over 
two benchmark datasets (MNIST dataset with LeNet architecture and CIFAR10 dataset with ResNet18 architecture) and 
three sets of the transformations.


\paragraph{Transformation Functions}
We consider three sets of transformation functions: 
\begin{itemize}
    \item \textbf{Texture}: we use Fourier transform to perturb the texture of the data by discarding the high-frequency components cut-off by a radius $r$, following \citep{wang2020high}. The smaller $r$ is, the fewer high-frequency components the image has. We consider $$\ma=\{a_0(), a_{12}(), a_{10}(), a_{8}(), a_{6}()\}$$, where the subscript denotes the radius $r$ except that $a_0()$ is the identity map. 
    We consider $\ma$ during test time, but only $a_0()$ and $a_6()$ during training. 
    \item \textbf{Rotation}: we rotate the images clockwise $r$ degrees with $\ma=\{a_0(), a_{15}(), a_{30}(), $ $a_{45}(), a_{60}()\}$, where the subscript denotes the degree of rotation and $a_0()$ is the identity map. 
    We consider $\ma$ during test time, but only $a()_0$ and $a_{60}()$ during training. 
    \item \textbf{Contrast}: we create the images depicting the same visual information, but with different scales of the pixels, including the negative color representation. 
    Therefore, we have $\ma = \{a_0(\x)=\x, a_1(\x) = \x/2, a_2(\x) = \x/4, a_3(\x) = 1-\x, a_4(\x) = (1-\x)/2, a_5(\x) = (1-\x)/4$, where $\x$ stands for the image whose pixel values have been set to be between 0 and 1. 
    We consider $\ma$ during test time, but only $a_0()$ and $a_3()$ during training.
\end{itemize}

\paragraph{Alignment Regularizations}
We consider the following popular choices of AR (with $\mathbf{u}$ and $\mathbf{v}$ denoting two vector embeddings): 
\begin{itemize}
    \item [\textsf{L}:] $\ell_1$ norm of the vector differences, \textit{i.e.}, $\Vert \mathbf{u} - \mathbf{v} \Vert_1$
    \item [\textsf{S}:] squared $\ell_2$ norm of the vector differences, \textit{i.e.}, $\Vert \mathbf{u} - \mathbf{v} \Vert_2^2$
    \item [\textsf{C}:] cosine similarity, \textit{i.e.}, $\mathbf{u}^T\mathbf{v}/\Vert\mathbf{u}\Vert\cdot\Vert\mathbf{v}\Vert$
    \item [\textsf{K}:] KL divergence over a batch of paired embeddings; the second argument are augmented samples. 
    \item [\textsf{W}:] Wasserstein-1 metric over a batch of paired embeddings, with implementation following Wasserstein GAN \citep{arjovsky2017wasserstein,gulrajani2017improved}
    \item [\textsf{D}:] a vanilla GAN discriminator over a batch of paired embeddings, the one-layer discriminator is trained to classify samples vs. augmented samples. 
\end{itemize}
We mainly discuss applying the AR to logits (embeddings prior to the final softmax function). We have also experimented with applying to the final softmax output and the embeddings one layer prior to logits. Both cases lead to substantially worse results, so we skip the discussion. 

\paragraph{Hyperparameters} 
We use standard data splits. 
We first train the baseline models to get reasonably high performances
(for MNIST, we train 100 epochs with learning rate set to be $10^{-4}$; 
for CIFAR10, we train 200 epochs with learning rate initialized to be $10^{-1}$ 
and reduced one magnitude every 50 epochs; 
batch sizes in both cases are set to be 128).  
Then we train other augmented models with the same learning rate and batch size \textit{etc}. 
The regularization weight is searched with 8 choices evenly split in the logspace from $10^{-7}$ to $1$. 
For each method, the reported score is from the weight resulting in
the highest robust accuracy. 
We test with three random seeds. 

\paragraph{Evaluation Metrics:} We consider the three major evaluation metrics as we discussed in Section~\ref{sec:background}: 
\begin{itemize}
    \item \textbf{Accuracy:} test accuracy on the original test data. 
    \item \textbf{Robustness:} the worst-case accuracy when each sample can be transformed with $a \in \ma$.  
    \item \textbf{Invariance:} the metric to test whether the learned representations are invariant to the transformation functions, as introduced in Section~\ref{sec:background:invariance}. 
\end{itemize}

\subsection{Results and Discussion}

Tables~\ref{tab:mnist:consistency} and~\ref{tab:cifar10:consistency} show the empirical results across the three distribution shifts and the three evaluation metrics. 
First of all, 
no method can dominate across all these 
evaluations, 
probably because of the tradeoff
between accuracy and robustness 
\citep{tsipras2018robustness,ZhangYJXGJ19,wang2020high}. 
Similarly, the tradeoff 
between accuracy and invariance
can be expected from the role of the regularization weight: 
when the weight is small, 
AR has no effect, 
and the model is primarily optimized for improving accuracy;
when the weight is considerable, 
the model is pushed toward a trivial solution 
that maps every sample to the same embedding, 
ignoring other patterns of the data. 
This is also the reason that results in 
Tables~\ref{tab:mnist:consistency} and~\ref{tab:cifar10:consistency}
are selected according to the robust accuracy. 

Due to the tradeoffs, it may not be strategic 
if we only focus on the highest number of each row. 
Instead, we suggest studying the three rows 
of each test case together and compare the tradeoffs, 
which is also a reason we reformulate the 
invariance test in Section~\ref{sec:background:invariance} 
so that we can have
bounded invariance scores directly comparable to 
accuracy and robust accuracy.

For example, 
in the texture rows of Table~\ref{tab:mnist:consistency}, 
while cosine similarity can outperform 
squared \ltwo{} norm in accuracy and robustness 
with 0.3 and 0.1 margins, respectively, 
it is disadvantageous in invariance with a larger margin (1.0).
Similarly, for rotation rows of Table~\ref{tab:mnist:consistency}, 
KL-divergence shows the overall highest scores, followed by $\ell_1$ norm and squared $\ell_2$ norm. 
For contrast rows, 
both $\ell_1$ norm and squared $\ell_2$ norm stand out. 
Overall, experiments in MNIST suggest the desired 
choice to be $\ell_1$ norm or squared $\ell_2$ norm, 
and we believe squared $\ell_2$ norm is marginally better. 

On the other hand, the experiments in CIFAR10 in Table~\ref{tab:cifar10:consistency}
mostly favor $\ell_1$ norm and squared $\ell_2$ norm. 
Good performances can also be observed from 
Wasserstein-1 metric for rotation. 
Also, we notice that squared $\ell_2$ norm, in general, outperforms 
$\ell_1$ norm. 

Thus, our empirical study recommends 
squared $\ell_2$ norm for AR to learn robust and invariant models, 
with $\ell_1$ norm as a runner-up.



\section{Analytical Support}
\label{sec:method}

According to our discussions in Section~\ref{sec:background:invariance}, 
Wasserstein metric is supposed to be a favorable option 
as AR. 
However, its empirical performance does not 
stand out.
We believe this disparity is mainly due to 
the difficulty in calculating the Wasserstein metric
in practice. 

On the other hand, 
norm-based consistency regularizations stand out. 
We are interested in studying the properties 
of these metrics. 
In particular, this section aims to complement the 
empirical study by showing that
norm-based AR can lead to bounded robust generalization error
under certain assumptions.  

Also, our experiments only use two transformation functions 
from $\ma$ during training but are tested with all the functions in $\ma$. 
We also discuss the properties of these two functions 
and argue that 
training with only these two functions can be a good strategy 
when certain assumptions are met. 


\subsection{Overview of Analytical Results}

All together, we need six assumptions (namely, A1-A6). 
Out of these assumptions, 
A1 is necessary for using data augmentation, 
and A4 is necessary for deriving machine learning generalization error bound. 
A2 and A3 are properties of data transformation functions, 
and A5 and A6 are technical assumptions. 
Along with introducing these assumptions, we will also offer empirical evidence showing that these assumptions are likely to hold in practice.

In particular, 
we will show that
\begin{itemize}
    \item With A2 holds, $\ell_1$ norm can replace the empirical Wasserstein metric to regularize invariance (Proposition~\ref{theory:proposition:l1}) and then we can derive a bounded robust error if all the functions in $\ma$ are available (Theorem~\ref{theory:theorem:main}). 
    \item With the above result and A3, we can derive a bounded robust error if only two special functions in $\ma$ (which we refer to as \emph{vertices}) are available (Lemma~\ref{theory:lemma:rv}).
\end{itemize}

\subsection{Assumptions Setup}

\label{sec:app:assumption}

\subsubsection{Assumptions on Data Augmentation Functions}

Our first three assumptions are for the basic properties of the data transformation functions used. 
These properties are formally introduced as assumptions below. 

\begin{itemize}
\item [\textbf{A1}:] \textbf{Dependence-preservation:}
    \textit{
    the transformation function will not alter the dependency regarding the label (\textit{i.e.}, for any $a()\in \ma$, $a(\x)$ will have the same label as $\x$)
    or the features (\textit{i.e.}, $a_1(\x_1)$ and $a_2(\x_2)$ are independent if $\x_1$ and $\x_2$ are independent).}
\end{itemize}
Intuitively, ``Dependence-preservation'' has two perspectives:  
Label-wise, the transformation cannot alter the label of the data, 
which is a central requirement of almost all the data augmentation functions in practice. 
Feature-wise, the transformation will not introduce new dependencies between the samples. 
            
We consider the label-wise half of this argument as a fundamental property of any data augmentations. It has to be always true for data augmentation to be a useful technique. On the other hand, the feature-wise half of this argument is a fundamental property required to derive the generalization error bounds. Intuitively, we believe this property holds for most data augmentation techniques in practice. 

\begin{itemize}
\item [\textbf{A2}:] \textbf{Efficiency:} \textit{ for $\wt$ and 
    any $a()\in \ma$, $\faxt$ is closer to $\x$ than any other samples under a distance metric 
    $d_e(\cdot, \cdot)$, \textit{i.e.}, 
    \begin{align*}
        d_e(\faxt, \fxt)\leq\min_{\x'\in\X_{-\x}}d_e(\faxt, f(\x';\wt)).
    \end{align*}
    We define $d_e(\cdot, \cdot)$ to be $\ell_1$ norm. }
\end{itemize}

Intuitively, the efficiency property means the augmentation should only generate new samples of the same label as minor perturbations of the original one.
If a transformation violates this property,
there should exist other simpler transformations
that can generate the same target sample. 

\begin{itemize}
\item [\textbf{A3}:] \textbf{Vertices:} \textit{
    For a model $\wt$ and a transformation $a()$, we use $\mP_{a, \wt}$ to denote the distribution of $\faxt$ for $(\x, \y)\sim \mP$. 
    ``Vertices'' argues that exists two extreme elements in $\ma$, namely $a^+$ and $a^-$, with certain metric $d_x(\cdot, \cdot)$, we have
    \begin{align*}
        d_x (\mP_{a^+, \wt}, \mP_{a^-, \wt}) = \sup_{a_1, a_2\in \ma} d_x (\mP_{a_1, \wt}, \mP_{a_2, \wt})
    \end{align*}
    We define $d_x(\cdot, \cdot)$ to be Wasserstein-1 metric. }
\end{itemize}
        
Intuitively, ``Vertices'' suggests that there are extreme cases of the transformations. 
For example, if one needs the model to be invariant to rotations from $0^{\circ}$ to $60^{\circ}$,
we consider the vertices to be $0^{\circ}$ rotation function (thus identity map) and $60^{\circ}$ rotation function. 
In practice, one usually selects the transformation vertices with intuitions or domain knowledge. 
    
Notice that we do not need to argue that \textbf{A3} always holds. 
All we need is that \textbf{A3} can sometimes hold, 
and when it holds, we can directly train with the regularized vertex augmentation. Thus, anytime \ra{} empirically performs well is a favorable argument for A3. 
To show that \ra{} can sometimes perform well, 
we compare the \ra{} with vanilla (non-regularized) worst-case data augmentation (\vwa{}) method across our synthetic experiment setup. 
We notice that 
out of six total scenarios (\{texture, rotation, contrast\} $\times$ \{MNIST, CIFAR10\}), \ra{} outperforms \vwa{} frequently. This suggests that the domain-knowledge of vertices can help in most cases, although not guaranteed in every case.

\subsubsection{Assumptions on Background and Generalization Error Bound}

In an abstract manner, when the test data and train data are from the same distribution, 
several previous analyses on the generalization error can be sketched as (see examples in \textbf{A4}):
\begin{align}
    \rp(\wt) \leq \wrp(\wt) + \phi(|\Theta|, n, \delta)
    \label{eq:standard}
\end{align}
which suggests that the expected risk $\rp(\wt)$ can be bounded by the empirical risk $\wrp(\wt)$
and a function of hypothesis space $|\Theta|$ and number of samples $n$; 
$\delta$ accounts for the probability when the bound holds. 
$\phi()$ is a function of these three terms. 
Dependent on the details of different analyses, 
different concrete examples of this generic term will need different assumptions. 
We use a generic assumption \textbf{A4} to denote 
the assumptions required for each example. 

Following our main goal to study 
how alignment regularization and data augmentation
help in 
accuracy, robustness, and invariance, 
our strategy in theoretical analysis 
is to derive error bounds for accuracy and robustness,
and the error bound directly contains terms to 
regularize the invariance. 
Further, as robustness naturally bounds accuracy
(\textit{i.e.}, $\rp(\wt)\leq\rpa(\wt)$ 
following the definitions in \eqref{eq:accuracy} and \eqref{eq:robustness} respectively), 
we only need to study the robust error. 

To study the robust error, we need two additional technical assumptions. 
\textbf{A5} connects the distribution of expected robust risk 
and the distribution of empirical robust risk, 
and \textbf{A6} connects the 0-1 classification error and cross-entropy error. 

\begin{itemize}
\item [\textbf{A4}:] \textit{We list two examples here: 
\begin{itemize}
    \item when \textbf{A4} is ``$\Theta$ is finite, $l(\cdot, \cdot)$ is a zero-one loss, samples are \textit{i.i.d}'',  $\phi(|\Theta|, n, \delta)=\sqrt{(\log(|\Theta|) + \log(1/\delta))/2n}$
    \item when \textbf{A4} is ``samples are \textit{i.i.d}'', $\phi(|\Theta|, n, \delta) = 2\mathcal{R}(\mathcal{L}) + \sqrt{(\log{1/\delta})/2n}$, where $\mathcal{R}(\mathcal{L})$ stands for Rademacher complexity and $\mathcal{L} = \{l_\theta \,|\, \theta \in \Theta \}$, where $l_\theta$ is the loss function corresponding to $\theta$. 
\end{itemize} }
\end{itemize}
For more information or more concrete examples of the generic term, 
one can refer to relevant textbooks such as \citep{bousquet2003introduction}. 

A4 stands for the fundamental assumptions used to derive standard generalization bounds. We rely on this assumption as how previous theoretical works rely on them. 

\begin{itemize}
\item[\textbf{A5}:] \textit{the distribution for expected robust risk 
equals the distribution for 
empirical robust risk, \textit{i.e.}, 
\begin{align*}
    \argmax_{\mathcal{P}'\in T(\mP, \ma)} \rpp(\wt) 
    = \argmax_{\mathcal{P'}\in T(\mP, \ma)}\wrpp(\wt) 
\end{align*}
where $T(\mP, \ma)$ is the collection of distributions created by elements in $\ma$ over samples from $\mP$. }
\end{itemize}
Eq. \eqref{eq:robustness} can be written equivalently into the expected risk 
over a pseudo distribution $\mP'$ (see Lemma 1 in \citep{tu2019theoretical}), 
which is the distribution that can sample the data leading to the expected robust risk. 
Thus, equivalently, we can consider 
$\sup_{\mathcal{P}'\in T(\mP, \ma)} \rpp(\wt)$ as a surrogate of $\rpa(\wt)$, where $T(\mP, \ma)$ denotes the set of possible resulting distributions. 
Following the empirical strength of techniques such as adversarial training \citep{MadryMSTV18}, 
we introduce an assumption relating the distribution of expected robust risk 
and the distribution of empirical robust risk 
(namely, \textbf{A5}, in  Appendix~\ref{sec:app:assumption}).
Thus, 
the bound of our interest (\textit{i.e.}, $\sup_{\mP'\in T(\mP, \ma)}\rpp(\wt)$) 
can be analogously analyzed through $\sup_{\mP'\in T(\mP, \ma)}\wrpp(\wt)$. 
    
\textbf{A5} is likely to hold in practice: Assumption \textbf{A5} appears very strong, 
however, the successes of methods like adversarial training \citep{MadryMSTV18} 
suggest that, in practice, 
\textbf{A5} might be much weaker 
than it appears. 

\begin{itemize}
    \item [\textbf{A6}:] \textit{ With $(\x,\y) \in (\X,\Y)$, the sample maximizing cross-entropy loss and the sample maximizing classification error for model $\wt$ follows: 
    \begin{align}
        \forall \x, \quad \dfrac{\y^\top \fxt}{\inf_{a\in \ma}\y^\top \faxt} \geq \exp\big(\mathbb{I}(g(\fxt)\neq g(f(\x';\wt)))\big)
        \label{eq:a2}
    \end{align}
    where $\x'$ stands for the sample maximizing classification error, \textit{i.e.}, 
    \begin{align*}
        \x' = \argmin_\x \y^\top g(\fxt)
    \end{align*}
    Also, 
    \begin{align}
        \forall \x, \quad \vert \inf_{a\in \ma}\y^\top \faxt \vert \geq 1
    \label{eq:assum:lipschitz}
    \end{align} }
  \end{itemize}  
Intuitively, although Assumption \textbf{A6} appears complicated, it describes the situations of two scenarios: 

If $g(\fxt)=g(f(\x';\wt))$, which means either the sample is misclassified by $\wt$ or $\ma$ is not rich enough for a transformation function to alter the prediction, the RHS of Eq.~\ref{eq:a2} is 1, thus Eq.~\ref{eq:a2} always holds (because $\ma$ has the identity map as one of its elements). 

If $g(\fxt)\neq g(f(\x';\wt))$, which means a transformation alters the prediction. In this case, \textbf{A6} intuitively states that the $\ma$ is reasonably rich and the transformation is reasonably powerful to create a gap of the probability for the correct class between the original sample and the transformed sample. The ratio is described as the ratio of the prediction confidence from the original sample over the prediction confidence from the transformed sample is greater than $e$.
        

\subsection{Analytical Support}

\paragraph{Regularized Worst-case Augmentation}

To have a model with a small invariance score, we should probably directly regularize the empirical counterpart of Eq.~\eqref{eq:invariance}. 
However, Wasserstein distance is difficult to calculate 
in practice. 
Fortunately, 
Proposition~\ref{theory:proposition:l1}
conveniently allows us to use $\ell_1$ norm to replace Wasserstein metric. 
With Proposition~\ref{theory:proposition:l1}, now
we can offer our main technical result to study the robust error $\rpa{\wt}$ (as defined in Eq.~\eqref{eq:robustness}).
\begin{theorem}
With Assumptions A1, A2, A4, A5, and A6, with probability at least $1-\delta$, 
we have
\begin{align*}
    \rpa{\wt}  \leq
    \wrp (\wt) + 
    \sum_{i}||f(\x_i;\wt) - f(\x'_i;\wt)||_1 + \phi(|\Theta|, n, \delta)
\end{align*}
and 
$\x' = a(\x) $, where $ a = \argmin_{a \in \ma} \y^\top \faxt$.
$\phi(|\Theta|, n, \delta)$ is defined in A4. 
\label{theory:theorem:main}
\end{theorem}

This technical result immediately inspires the method to guarantee worst case performance, 
as well as to explicitly enforce the concept of invariance. 
The method 
$a = \argmin_{a \in \ma} \y^\top \faxt$ is selecting the transformation function 
maximizing the cross-entropy loss 
(notice the sign difference between here and the cross-entropy loss), 
which we refer to as worst-case data augmentation.
This method is also closely connected to adversarial training \citep[\textit{e.g.},][]{MadryMSTV18}. 

\paragraph{Regularized Vertex Augmentation}

As $\ma$ in practice usually has a large number of
(and possibly infinite) elements, we may not always be able to identify 
the worst-case transformation function with reasonable computational efforts. 
We further leverage the vertex property (boundary cases of transformation functions, discussed as Assumption A3 in the appendix) of the transformation function 
to bound the worst-case generalization error:
\begin{lemma}
With Assumptions A1-A6, 
assuming there is a $a'()\in \ma$ where $ \widehat{r}_{\mP_{a'}}(\wt)=\frac{1}{2}\big(\widehat{r}_{\mP_{a^+}}(\wt)+\widehat{r}_{\mP_{a^-}}(\wt)\big)$,
with probability at least $1-\delta$, we have:
\begin{align*}
    \rpa(\wt) \leq &
    \dfrac{1}{2}\big(\widehat{r}_{\mP_{a^+}}(\wt) + \widehat{r}_{\mP_{a^-}}(\wt)\big) \\
    &+ \sum_{i}||f(a^+(\x_i);\wt) - f(a^-(\x');\wt)||_1 
    + \phi(|\Theta|, n, \delta), 
\end{align*}
where $a^+()$ and $a^-()$ are defined in A3, 
and $\phi(|\Theta|, n, \delta)$ in A4.
\label{theory:lemma:rv}
\end{lemma}

This result corresponds to the method that 
can be optimized conveniently 
without searching for the worst-case transformations. 
However, the method requires good domain knowledge of the vertices 
(\textit{i.e.}, boundary cases)
of the transformation functions.

\begin{table*}[t]
\small 
\centering 
\caption{Comparison to advanced rotation-invariant models. We report the accuracy on the test sets rotated. ``\textsf{main}'' means the resulting images are highly likely to be semantically the same as the original ones. ``\textsf{all}'' means the average accuracy of all rotations. 
The underlined scores show that data augmentation 
and AR can help a vanilla model to compete with 
advanced methods. 
The bold scores (highest at each row) show that data augmentation 
and AR can further improve the advanced methods. 
}

\begin{tabular}{c|ccc|ccc|ccc|ccc}
\hline
 & \multicolumn{3}{c|}{\textsf{ResNet}} & \multicolumn{3}{c|}{\textsf{GC}} & \multicolumn{3}{c|}{\textsf{ST}} & \multicolumn{3}{c}{\textsf{ETN}} \\ 
 & \textsf{Base} & \ra{} & \rwa{} & \textsf{Base} & \ra{} & \rwa{} & \textsf{Base} & \ra{} & \rwa{} & \textsf{Base} & \ra{} & \rwa{} \\ \hline
\textsf{main} & 45.4 & 66.5 & {\ul 71.1} & 38.5 & 72.2 & \textbf{73.8} & 45.9 & 58.3 & 62.9 & 56.9 & 65.1 & 57.7 \\ \hline
\textsf{all} & 31.2 & 48.1 & {\ul 52.8} & 26.7 & 54.4 & \textbf{55.0} & 32.1 & 40.2 & 42.7 & 39.5 & 52.6 & 46.1 \\ \hline
\end{tabular}
\label{tab:real:rotation}
\end{table*}

\begin{table}[]
\small 
\centering 
\caption{Comparison to advanced methods on 9 super-class ImageNet classification with different distribution shifts.}
\begin{tabular}{ccccc}
\hline
 & Acc. & WAcc. & ImageNet-A & ImageNet-S \\ \hline
\textsf{Base} & 90.8 & 88.8 & 24.9 & 41.1 \\
\textsf{SIN} & 88.4 & 86.6 & 24.6 & 40.5 \\
\textsf{LM} & 67.9 & 65.9 & 18.8 & 36.8 \\
\textsf{RUBi} & 90.5 & 88.6 & 27.7 & 42.3 \\
\textsf{RB} & 91.9 & 90.5 & \textbf{29.6} & 41.8 \\
\ra{} & 92.2 & 91.2 & 28.0 & 42.5 \\
\rwa{} & \textbf{92.8} & \textbf{91.6} & 28.8 & \textbf{43.2} \\ \hline
\end{tabular}
\label{tab:real:texture}
\end{table}
\begin{table}[]
\centering
\caption{The generic methods can also improve standard accuracy. }
\small 
\begin{tabular}{cccccccccc}
\hline
\multirow{2}{*}{} & \multicolumn{3}{c}{ResNet18} & \multicolumn{3}{c}{ResNet50} & \multicolumn{3}{c}{ResNet101} \\
 & \base{} & \ra{} & \rwa{} & \base{} & \ra{} & \rwa{} & \base{} & \ra{} & \rwa{} \\ \hline
Top-1 & 75.6 & 100 & 77.2 & 77.4 & 100 & 78.2 & 77.8 & 100 & 78.7 \\
Top-5 & 93.1 & 100 & 93.8 & 93.9 & 100 & 94.4 & 94.4 & 100 & 94.9 \\ \hline
\end{tabular}
\label{tab:accuracy}
\end{table}

Thus, our theoretical discussions have complemented 
our empirical findings in Section~\ref{sec:motivate}
by showing that norm-based regularizations 
can lead to bounded robust error. 
There is a disparity that 
our analytical result is about 
$\ell_1$ norm
while our empirical study suggests 
squared \ltwo{} norm. 
We conjecture the disparity is mainly caused 
by the difficulty in passing the gradient 
of $\ell_1$ norm in practice.

\section{Experiments with Advanced Methods}
\label{sec:exp}

We continue to test the methods we identified 
in comparison to more advanced methods. 
Although we argued for the value of invariance, 
for a fair comparison, we will 
test the performances evaluated by the metrics
the previous methods are designed for. 
Our method will use the same generic approach
and the same transformation functions as in the previous
empirical study, 
although these functions are not necessarily part of the  
distribution shift we test now. 
In summary, 
our method can outperform (or be on par with) these SOTA techniques 
in the robustness metric they are designed for (Section~\ref{sec:exp:robust}). 
In addition, 
we run a side test to show that 
our method can also 
improve accuracy (Section~\ref{sec:exp:accuracy}). 

\subsection{Methods}

Section~\ref{sec:motivate} and Section~\ref{sec:method} lead us to test the following two methods: 
\begin{itemize}
    \item \ra{} (regularized vertex augmentation): using squared $\ell_2$ norm as AR over logits between the original samples and the augmented samples of a fixed vertex transformation function (original samples are considered as from another vertex). 
    \item \rwa{} (regularized worst-case augmentation): using squared $\ell_2$ norm as AR over logits between the original samples and the worst-case augmented samples identified at each iteration. Worst-case samples are generated by the function with the maximum loss when we iterate through all the transformation functions.
\end{itemize}

\subsection{Robustness}
\label{sec:exp:robust}

\paragraph{Rotation}
We compare our results with rotation-invariant models,
mainly Spatial Transformer (\textsf{ST}) \citep{jaderberg2015spatial},
Group Convolution (\textsf{GC}) \citep{cohen2016group}, 
and Equivariant Transformer Network (\textsf{ETN}) \citep{tai2019equivariant}. 
We also tried to run CGNet \citep{kondor2018clebsch}, 
but the method does not seem to scale to the CIFAR10 and ResNet level. 
All these methods are tested with ResNet34 following popular settings in the community. 
The results are in Table~\ref{tab:real:rotation}. 
We test the models every $15^{\circ}$ rotation from $0^{\circ}$  rotation to $345^{\circ}$ rotation. 
Augmentation-related methods use the $\ma$ of ``rotation''
in synthetic experiments, 
so the testing scenario goes beyond what the augmentation methods have seen during training. 

We report two summary results in Table~\ref{tab:real:rotation}. 
``\textsf{main}'' means the average prediction accuracy from images rotated from $300^{\circ}$ to $60^{\circ}$ (passing $0^{\circ}$), 
when the resulting images are highly likely to preserve the class label. ``\textsf{all}'' means the average accuracy of all rotations. 

Our results can be interpreted from two perspectives. 
First, by comparing all the columns in the first panel 
to the first column of the other three panels, 
data augmentation and AR 
can boost a vanilla model 
to outperform other advanced techniques. 
On the other hand, 
by comparing the columns within each panel, 
data augmentation and AR 
can further improve the performances of these techniques. 

Interestingly, the baseline model 
with our generic approach (\rwa{} in the first panel)
can almost compete 
with the advanced methods even when these methods also use augmentation 
and AR (\rwa{} in \textsf{GC} panel). 
We believe this result strongly indicates 
the potential of this simple augmentation and regularization method to match the advanced methods.

In summary, \rwa{} can boost
the vanilla model to outperform advanced methods. 
Data augmentation and squared $\ell_2$ AR can further improve the performances
when plugged onto advanced methods. 

\paragraph{Texture \& Contrast}
We follow \citep{bahng2019learning} 
and compare the models for a nine super-class ImageNet classification \citep{ilyas2019adversarial} with class-balanced strategies. 
Also, we follow \citep{bahng2019learning}
to report standard accuracy (Acc.), 
weighted accuracy (WAcc.), a scenario where
samples with unusual texture are weighted more, 
and accuracy over ImageNet-A \citep{hendrycks2019natural},
a collection of 
failure cases for most ImageNet trained models. 
Additionally, we also report 
the performance over ImageNet-Sketch \citep{wang2019learning}, 
an independently collected ImageNet test set
with only sketch images. 
As \citep{bahng2019learning} mainly aims to overcome 
the texture bias, 
we also use our texture-wise functions 
in Section~\ref{sec:motivate} for augmentation. 
However, there are no direct connections between 
these functions and the distribution shift
of the test samples. 
Also, we believe the distribution shifts here, 
especially the one introduced by our newly added ImageNet-Sketch, 
are more than texture,
and also correspond to the contrast case of our study. 

Following \citep{bahng2019learning}, 
the base network is ResNet, 
and we compare with the vanilla network (\textsf{Base}), 
and several methods designed for this task:
including 
StylisedIN (\textsf{SIN}) \citep{geirhos2018imagenettrained}, 
LearnedMixin (\textsf{LM}) \citep{clark2019don}, 
RUBi (\textsf{RUBi}) \citep{cadene2019rubi}
and ReBias (\textsf{RB}) \citep{bahng2019learning}. 
Results are in Table~\ref{tab:real:texture}. 

The results favor our generic method in most cases. 
\ra{} outperforms other methods in standard accuracy, weighted accuracy, and ImageNet-Sketch, and is shy from ReBias on ImageNet-A. 
\rwa{} shows the same pattern as that of \ra{} and further outperforms \ra{}. 
Overall, these results validate the empirical strength of
data augmentation (even when the augmentation is not designed for the task) and squared $\ell_2$ norm AR for learning robust models. 

\subsection{Accuracy}
\label{sec:exp:accuracy}

Further, 
these experiments help us notice that
the generic technique can also help improve the accuracy, 
although the technique is motivated by robustness and invariance.
Therefore, 
we follow the widely accepted \href{https://github.com/weiaicunzai/pytorch-cifar100}{CIFAR100 test pipeline}
and test the performances of different architectures
of the ResNet family. 
The results are reported in Table~\ref{tab:accuracy}, 
where \textsf{Base} stands for the baseline model with the default accuracy boosting configurations. 

For both top-1 and top-5 accuracies and across the three ResNet architectures, 
our techniques can help improve the accuracy. 
In addition, 
we notice that 
our techniques can help bridge the gap of 
different architectures within the ResNet family:
for example, \rwa{} helps ResNet50 to outperform 
the vanilla ResNet101.

\section{Conclusion}
\label{sec:con}

In this paper, we seek to answer how to train with augmented data so that augmentation can be taken to the fullest extent. We first defined a new evaluation metric called invariance and conducted a line of empirical studies to show that norm-based alignment regularization can help learn robust and invariant models. 
Further, we complement our observations with formal derivations
of bounded generalization errors. 
We notice that 
regularizing squared $\ell_2$ norm between the logits of the originals samples and those of the augmented samples is favorable:
the trained model tends to have the most favorable performances in robust accuracy and invariance. 
In general, the method we recommend is 
``regularized worst-case augmentation'' with squared $\ell_2$ norm as the alignment regularization. 
One can also consider ``regularized vertex augmentation'' when extra assumptions
on the vertex properties of the transformation functions are met. 
Lastly, we would like to remind a potential limitation of alignment regularization: 
it may not always help improve the \textit{i.i.d} accuracy due to the tradeoff between accuracy and robustness or invariance. 
In addition, 
to simplify the procedure of users in leveraging our contribution, 
we also release a software package 
in both \textsf{TensorFlow} and \textsf{PyTorch} 
for users to use our identified methods 
with a couple lines of code.



\begin{acks}
This work was supported by NIH R01GM114311, NIH P30DA035778, and NSF IIS1617583; NSF CAREER IIS-2150012 and IIS-2204808. 
The authors would like to thank Hanru Yan for the implementation of the software package. 
\end{acks}

\bibliographystyle{ACM-Reference-Format}
\bibliography{ref}


\clearpage 
\appendix

\newpage

\section*{Appendices}

\section{Proof of Theoretical Results}

\subsection{Lemma A.1 and its proof}
\begin{lemma}
With Assumptions A1, A4, and A5, with probability at least $1-\delta$, we have 
\begin{align*}
    \rpa(\wt)  \leq 
    \dfrac{1}{n}\sum_{(\x, \y) \sim \mP}\sup_{a\in \ma}\mI(g(\faxt) \neq \y)  + \phi(|\Theta|, n, \delta)
\end{align*}
\end{lemma}

\begin{proof}
With Assumption A5, we have
\begin{align*}
    \argmax_{\mP'\in T(\mP, \ma)} \rpp(\wt) 
    = \argmax_{\mP'\in T(\mP, \ma)}\wrpp(\wt) = \mP_w
\end{align*}
we can analyze the expected risk following the standard classical techniques since both expected risk and empirical risk are studied over distribution $\mP_w$. 

Now we only need to make sure the classical analyses (as discussed in A4) are still valid over distribution $\mP_w$:
\begin{itemize}
    \item when \textbf{A4} is ``$\Theta$ is finite, $l(\cdot, \cdot)$ is a zero-one loss, samples are \textit{i.i.d}'',  $\phi(|\Theta|, n, \delta)=\sqrt{\dfrac{\log(|\Theta|) + \log(1/\delta)}{2n}}$. 
    The proof of this result uses Hoeffding's inequality, which only requires independence of random variables. One can refer to Section 3.6 in \cite{liang2016cs229t} for the detailed proof. 
    \item when \textbf{A4} is ``samples are \textit{i.i.d}'', $\phi(|\Theta|, n, \delta) = 2\mathcal{R}(\mathcal{L}) + \sqrt{\dfrac{\log{1/\delta}}{2n}}$. 
    The proof of this result relies on McDiarmid's inequality, which also only requires independence of random variables. One can refer to Section 3.8 in \cite{liang2016cs229t} for the detailed proof. 
\end{itemize}
Assumption \textbf{A1} guarantees the samples from distribution $\mP_w$ are still independent, thus the generic term holds for at least these two concrete examples, thus the claim is proved. 

\end{proof}

\subsection{Proposition A.2 and Proof}
\begin{proposition}
With A2, 
for any $a\in \ma$, we have
\begin{align*}
    W_1(\widehat{\mathbf{Q}}_{\x, \wt}, \widehat{\mathbf{Q}}_{a(\x), \wt})=\sum_{i}^{|(\X, \Y)|} ||\fxit-\faxit||_1, 
\end{align*}
where $\widehat{\mathbf{Q}}_{\x, \wt}$ denotes the empirical 
distribution of $\fxt$ for $(\x, \y)\in (\X, \Y)$. 
\label{theory:proposition:l1}
\end{proposition}
\begin{proof}
We use the order statistics representation of Wasserstein metric over empirical distributions (\textit{e.g.}, see Section 4 in \cite{bobkov2019one})
\begin{align*}
    W_1(\widehat{\mathbf{Q}}_{\x, \wt}, \widehat{\mathbf{Q}}_{a(\x), \wt})) = \inf_{\sigma}\sum_{i}^{|(\X, \Y)|}||\fxit - f(a(\x_{\sigma(i)}), \wt)||_1
\end{align*}
where $\sigma$ stands for a permutation of the index, thus the infimum is taken over all possible permutations.
With Assumption A2, when $d_e(\cdot,\cdot)$ in A2 chosen to be $\ell_1$ norm, we have: 
\begin{align*}
    ||\fxit - \faxit||_1 \leq \min_{j\neq i} ||\fxit - f(a(\x_{j}), \wt)||_1
\end{align*}
Thus, the infimum is taken when $\sigma$ is the natural order of the samples, which leads to the claim. 
\end{proof}

\subsection{Proof of Theorem 5.1}
\begin{theorem*}
With Assumptions A1, A2, A4, A5, and A6, with probability at least $1-\delta$, we have
\begin{align}
    \rpa(\wt) \leq 
    \wrp (\wt) + 
    \sum_{i}||f(\x_i;\wt) - f(\x'_i;\wt)||_1 + 
    \phi(|\Theta|, n, \delta)
\end{align}
and 
$\x' = a(\x) $, where $ a = \argmin_{a \in \ma} \y^\top \faxt$.
\end{theorem*}

\begin{proof}

First of all, in the context of multiclass classification, where $g(f(\x, ;\theta))$ predicts a label with one-hot representation, and $\y$ is also represented with one-hot representation, we can have the empirical risk written as:
\begin{align*}
    \wrp(\wt) = 1 -
    \dfrac{1}{n}\sum_{(\x, \y) \sim \mP}\y^\top g(\fxt)
\end{align*}
Thus, 
\begin{align*}
    \sup_{\mP'\in T(\mP, \ma)}
    \wrpp(\wt) 
    = & \wrp(\wt) + 
    \sup_{\mP'\in T(\mP, \ma)}
    \wrpp(\wt) -
    \wrp(\wt) \\
    = & \wrp(\wt) + \dfrac{1}{n}\sup_{\mP'\in T(\mP, \ma)}\big(\sum_{(\x, \y) \sim \mP} \y^\top g(\fxt) \\
     & - \sum_{(\x, \y) \sim \mP'}\y^\top g(\fxt)\big)
\end{align*}

With A6, we can continue with: 
\begin{align*}
    \sup_{\mP'\in T(\mP, \ma)}
    \wrpp(\wt) 
    \leq & \wrp(\wt)
    + \dfrac{1}{n}\sup_{\mP'\in T(\mP, \ma)}\big(\sum_{(\x, \y) \sim \mP} \y^\top \log(\fxt) \\ &- \sum_{(\x, \y) \sim \mP'}\y^\top \log(\fxt)\big)
\end{align*}
If we use $e(\cdot)=-\y^\top \log(\cdot)$ to replace the cross-entropy loss, we simply have:
\begin{align*}
    \sup_{\mP'\in T(\mP, \ma)}
    \wrpp(\wt) 
    \leq & \wrp(\wt)
    + \dfrac{1}{n}\sup_{\mP'\in T(\mP, \ma)}\big(\sum_{(\x, \y) \sim \mP'} e(\fxt) \\ &- \sum_{(\x, \y) \sim \mP}e((\fxt)\big)
\end{align*}
Since $e(\cdot)$ is a Lipschitz function with constant $\leq 1$ (because of A6, Eq.\eqref{eq:assum:lipschitz}) and together with the dual representation of Wasserstein metric (See \textit{e.g.}, \cite{villani2003topics}), 
we have
\begin{align*}
    \sup_{\mP'\in T(\mP, \ma)}
    \wrpp(\wt) 
    & \leq \wrp(\wt)
    +  W_1(\widehat{\mathbf{Q}}_{\x, \wt}, \widehat{\mathbf{Q}}_{a(\x), \wt}))
\end{align*}
where $\x' = a(\x) $, where $ a = \argmin_{a \in \ma} \y^\top \faxt$; $\widehat{\mathbf{Q}}_{\x, \wt}$ denotes the empirical 
distribution of $\faxt$ for $(\x, \y)\in (\X, \Y)$.
Note that $\rpa(\wt)$, by definition, is a shorthand notation for $$\sup_{\mP'\in T(\mP, \ma)}\rpp(\wt)$$.

Further, we can use the help of Proposition B.2 to replace Wassertein metric with $\ell_1$ distance. 
Finally, we can conclude the proof with Assumption A5 as how we did in the proof of Lemma B.1.    
\end{proof}

\begin{table*}[]
\centering
{\tiny    
\begin{tabular}{cc|ccc|cccc|cccc|cccc|c}
\hline
\multirow{2}{*}{} & \multirow{2}{*}{Clean} & \multicolumn{3}{c|}{Noise} & \multicolumn{4}{c|}{Blur} & \multicolumn{4}{c|}{Weather} & \multicolumn{4}{c|}{Digital} & \multirow{2}{*}{mCE} \\ \cline{3-17}
 &  & Gauss & Shot & Impulse & Defocus & Glass & Motion & Zoom & Snow & Frost & Fog & Bright & Contrast & Elastic & Pixel & JPEG &  \\ \hline
\base{} & 23.9 & 79 & 80 & 82 & 82 & 90 & 84 & 80 & 86 & 81 & 75 & 65 & 79 & 91 & 77 & 80 & 80.6 \\
\ra{} & 23.6 & 78 & 78 & 79 & 74 & 87 & 79 & 76 & 78 & 75 & 69 & 58 & 68 & 85 & 75 & 75 & 75.6 \\
\rwa{} & 22.4 & 61 & 63 & 63 & 68 & 75 & 65 & 66 & 70 & 69 & 64 & 56 & 55 & 70 & 61 & 63 & 64.6 \\
\textsf{SU} & 24.5 & 67 & 68 & 70 & 74 & 83 & 81 & 77 & 80 & 74 & 75 & 62 & 77 & 84 & 71 & 71 & 74.3 \\
\textsf{AA} & 22.8 & 69 & 68 & 72 & 77 & 83 & 80 & 81 & 79 & 75 & 64 & 56 & 70 & 88 & 57 & 71 & 72.7 \\
\textsf{MBP} & 23 & 73 & 74 & 76 & 74 & 86 & 78 & 77 & 77 & 72 & 63 & 56 & 68 & 86 & 71 & 71 & 73.4 \\
\textsf{SIN} & 27.2 & 69 & 70 & 70 & 77 & 84 & 76 & 82 & 74 & 75 & 69 & 65 & 69 & 80 & 64 & 77 & 73.3 \\
\textsf{AM} & 22.4 & 65 & 66 & 67 & 70 & 80 & 66 & 66 & 75 & 72 & 67 & 58 & 58 & 79 & 69 & 69 & 68.4 \\
\textsf{AMS} & 25.2 & 61 & 62 & 61 & 69 & 77 & 63 & 72 & 66 & 68 & 63 & 59 & 52 & 74 & 60 & 67 & 64.9 \\ \hline
\end{tabular}
}
\caption{Comparison to advanced models over ImageNet-C data. Performance reported (mCE) follows the standard in ImageNet-C data: mCE is the smaller the better.}
\label{tab:real:c:app}
\end{table*}

\begin{table*}[]
\small
\centering 
\begin{tabular}{cccccccccc}
\hline
 & \base{} & \textsf{InfoDrop} & \textsf{HEX} & \textsf{PAR} & \va{} & \ra{} & \textsf{RSC} & \vwa{} & \rwa{}   \\ \hline
Top-1 & 0.1204 & 0.1224 & 0.1292 & 0.1306 & 0.1362& 0.1405&  0.1612 &0.1432 & 0.1486 \\
Top-5 & 0.2408 & 0.256 & 0.2564 & 0.2627 &0.2715 &0.2793  &  0.3078& 0.2846 & 0.2933\\ \hline
\end{tabular}
\caption{Comparison to advanced cross-domain image classification models, over ImageNet-Sketch dataset. We report top-1 and top-5 accuracy following standards on ImageNet related experiments. 
}
\label{tab:real:sketch:app}
\end{table*}

\subsection{Proof of Lemma 5.2}
\textbf{Lemma.}
\textit{
With Assumptions A1-A6, 
assuming there is a $a'()\in \ma$ where $ \widehat{r}_{\mP_{a'}}(\wt)=\frac{1}{2}\big(\widehat{r}_{\mP_{a^+}}(\wt)+\widehat{r}_{\mP_{a^-}}(\wt)\big)$,
with probability at least $1-\delta$, we have:
\begin{align}
    \rpa(\wt)  \leq &
    \dfrac{1}{2}\big(\widehat{r}_{\mP_{a^+}}(\wt) \\
    & + \widehat{r}_{\mP_{a^-}}(\wt)\big) + 
    \sum_{i}||f(a^+(\x_i);\wt) - f(a^-(\x');\wt)||_1 + 
    \phi(|\Theta|, n, \delta)
\end{align}
}
\begin{proof}
We can continue with 
\begin{align*}
    \sup_{\mP'\in T(\mP, \ma)}
    \wrpp(\wt) 
    & \leq \wrp(\wt)
    +  W_1(\widehat{\mathbf{Q}}_{\x, \wt}, \widehat{\mathbf{Q}}_{a(\x), \wt})), 
\end{align*}
where $\widehat{\mathbf{Q}}_{\x, \wt}$ denotes the empirical 
distribution of $\faxt$ for $(\x, \y)\in (\X, \Y)$.
from the proof of Theorem 5.2. 
With the help of Assumption A3, we have:
\begin{align*}
    d_x (f(a^+(\x), \wt), f(a^-(\x), \wt)) \geq d_x (f(\x, \wt), f(\x', \wt))
\end{align*}
When $d_x(\cdot, \cdot)$ is chosen as Wasserstein-1 metric, we have: 
\begin{align*}
    \sup_{\mP'\in T(\mP, \ma)}
    \wrpp(\wt) 
    & \leq \wrp(\wt)
    +  W_1(\widehat{\mathbf{Q}}_{a^+(\x), \wt}, \widehat{\mathbf{Q}}_{a^-(\x), \wt}))
\end{align*}
Further, as the LHS is the robust risk generated by the transformation functions within $\ma$, 
and $\wrp(\wt)$ is independent of the term $W_1(\widehat{\mathbf{Q}}_{a^+(\x), \wt}, \widehat{\mathbf{Q}}_{a^-(\x), \wt}))$, 
WLOG, we can replace $\wrp(\wt)$ with the risk of an arbitrary distribution generated by the transformation function in $\ma$. 
If we choose to use $ \widehat{r}_{\mP_{a'}}(\wt)=\frac{1}{2}\big(\widehat{r}_{\mP_{a^+}}(\wt)+\widehat{r}_{\mP_{a^-}}(\wt)\big)$, 
we can conclude the proof with help from Proposition B.2 and Assumption A5 as how we did in the proof of Theorem 5.2.  
\end{proof}

\section{Additional Results for Comparisons with Advanced Methods}
\label{sec:app:real}

We have also conducted two full ImageNet level experiments. However, due to the limitation of resources, we cannot tune the models substantially. 
Our current trial suggest that our techniques can improve the vanilla model to compete with SOTA models, limited by our resources, 
we cannot do wide-range hyperparameters search to outperform them. 
Also, considering the fact that many of these methods are significantly more complicated than us and also uses data augmentation specially designed for the tasks, we consider our experiments a success indication 
of the empirical strength of our methods.

\paragraph{Texture-perturbed ImageNet classification}
We also test the performance on the image classification over multiple perturbations. 
We train the model over standard ImageNet training set and test the model with ImageNet-C data \citep{hendrycks2019robustness}, which is a perturbed version of ImageNet by corrupting the original ImageNet validation set with a collection of noises. Following the standard, the reported performance is mCE, which is the smaller the better. 
We compare with several methods tested on this dataset, including 
Patch Uniform (\textsf{PU}) \citep{lopes2019improving}, 
AutoAugment (\textsf{AA}) \citep{cubuk2019autoaugment}, 
MaxBlur pool (\textsf{MBP}) \citep{zhang2019making}, 
Stylized ImageNet (\textsf{SIN}) \citep{hendrycks2019robustness}, 
AugMix (\textsf{AM}) \citep{Hendrycks2020augmix}, 
AugMix w. SIN (\textsf{AMS}) \citep{Hendrycks2020augmix}. 
We use the performance reported in \citep{Hendrycks2020augmix}. 
Again, our augmention only uses the generic texture with perturbation (the $\ma$ in our texture synthetic experiments with radius changed to $20, 25, 30, 35, 40$).
The results are reported in Table~\ref{tab:real:c:app}, which shows that
our generic method outperform the current SOTA methods after a continued finetuning process with reducing learning rates. 

\paragraph{Cross-domain ImageNet-Sketch Classification}
We also compare to the methods used for cross-domain evaluation. 
We follow the set-up advocated by \citep{wang2019learning2} 
for domain-agnostic cross-domain prediction, 
which is training the model on one or multiple domains 
without domain identifiers and test the model on an unseen domain. 
We use the most challenging setup in this scenario: 
train the models with standard ImageNet training data, 
and test the model over ImageNet-Sketch data \citep{wang2019learning}, which is a collection of sketches 
following the structure ImageNet validation set. 
We compare with previous methods with reported performance on this dataset, 
such as \textsf{InfoDrop} \citep{achille2018information}, 
\textsf{HEX} \citep{wang2019learning2}, \textsf{PAR} \citep{wang2019learning}, 
\textsf{RSC} \citep{HuangWXH20}
and report the performances in Table~\ref{tab:real:sketch:app}. 
Notice that, our data augmentation also follows the 
requirement that the characteristics of the test domain cannot 
be utilized during training. 
Thus, we only augment the samples with a generic augmentation set
($\ma$ of ``contrast'' in synthetic experiments). 
The results again support the usage of data augmentation and alignment regularization.  


\end{document}